\documentclass[10pt,twocolumn,letterpaper]{article}

\usepackage{iccv}
\usepackage{times}
\usepackage{epsfig}
\usepackage{graphicx}
\usepackage{amsmath}
\usepackage{amssymb}
\usepackage[title]{appendix}

\usepackage{amsthm}
\newtheorem{theorem}{Theorem}
\newtheorem{lemma}[theorem]{Lemma}
\usepackage{mathtools}
\usepackage{tabularx,booktabs}
\newcolumntype{Y}{>{\centering\arraybackslash}X}
\DeclareMathOperator*{\maxp}{max_{\mathnormal{}}}
\DeclareMathOperator*{\avgp}{avg_{\mathnormal{}}}

\usepackage[pagebackref=true,breaklinks=true,letterpaper=true,colorlinks,bookmarks=false]{hyperref}

\iccvfinalcopy 

\ificcvfinal\fi
\begin{document}

\title{PS$^2$-Net: A Locally and Globally Aware Network for Point-Based Semantic Segmentation}

\author{Na Zhao ~~~ Tat-Seng Chua ~~~ Gim Hee Lee \\
National University of Singapore\\
{\tt\small \{nazhao, chuats, leegh\}@comp.nus.edu.sg}
}

\maketitle

\begin{abstract}
	In this paper, we present the PS$^2$-Net - a locally and globally aware deep learning framework for semantic segmentation on 3D scene-level point clouds. In order to deeply incorporate local structures and global context to support 3D scene segmentation, our network is built on four repeatedly stacked encoders, where each encoder has two basic components: EdgeConv that captures local structures and NetVLAD that models global context. Different from existing start-of-the-art methods for point-based scene semantic segmentation that either violate or do not achieve permutation invariance, our PS$^2$-Net is designed to be permutation invariant which is an essential property of any deep network used to process unordered point clouds. We further provide theoretical proof to guarantee the permutation invariance property of our network.
	We perform extensive experiments on two large-scale 3D indoor scene datasets and demonstrate that our PS$^2$-Net is able to achieve state-of-the-art performances as  compared to existing approaches.
\end{abstract}

\vspace{-0.1in}
\section{Introduction}
Semantic scene segmentation refers to the process of assigning a class label to each element representation of a scene. The outcome of semantic scene segmentation is extremely useful for many applications in artificial intelligence such as interactions of robots (\eg self-driving cars and autonomous drones) with its environment, and augmented/virtual reality (AR/VR) \etc. Semantic scene segmentation on 2D images, where each pixel is the elementary representation of the scene, is a long-standing problem in computer vision; and many impressive results have been shown with deep learning \cite{badrinarayanan2017segnet, chen2018deeplab, he2017mask, long2015fully, ronneberger2015u} in recent years. In contrast to its 2D image counterpart, semantic scene segmentation on 3D point cloud, where each point is the elementary representation of the scene, has gained attention in the deep learning/ computer vision community only over last few years. This is largely attributed to the non-permutation invariance property of neural networks \cite{qi2017pointnet, zaheer2017deep}, which work well on 2D image pixels that are arranged in a regular orderly structure, but would fail catastrophically when used on 3D point clouds where the points exist in an unorderly and irregular manner. 

\begin{figure}[t]
	\centering
	\includegraphics[scale=0.35]{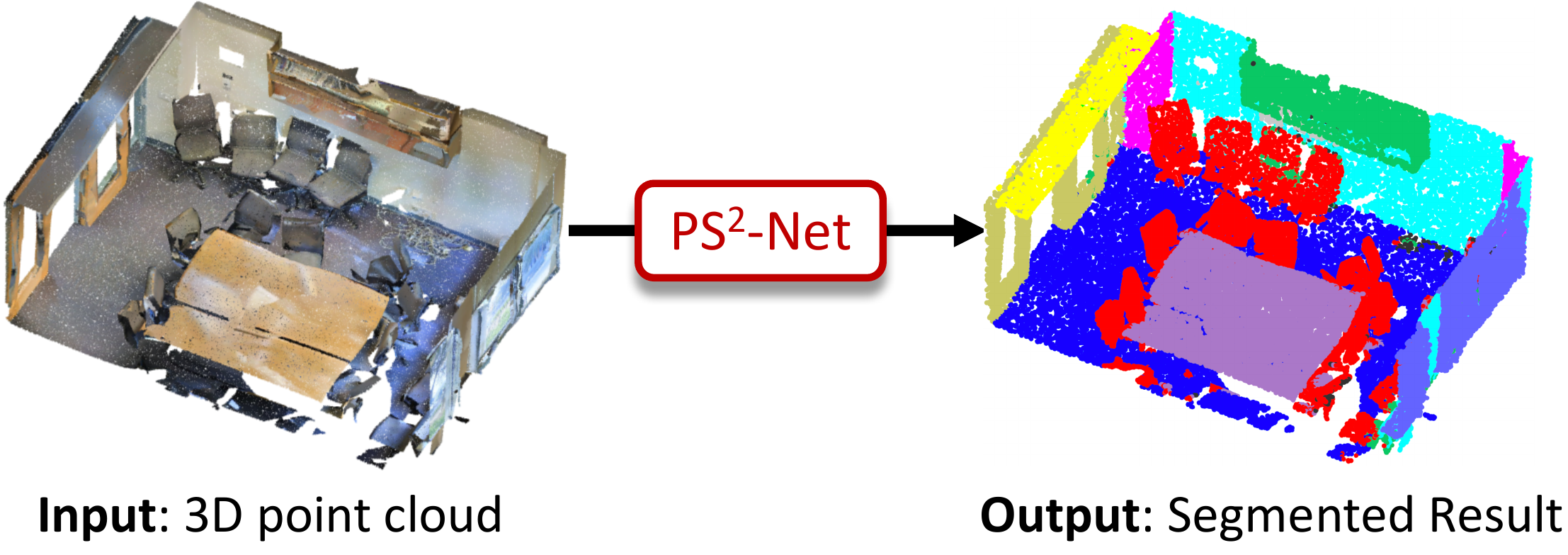}
	\caption{Our PS$^2$-Net segments raw point cloud into semantic homogeneous regions.}
	\label{fig:teaser}
	\vspace{-0.2in}
\end{figure}

To circumvent the permutation invariance problem, several existing works attempt to convert the unordered and irregular 3D point cloud into ordered and regular representations that can be directly used in neural networks. Two commonly used approaches include: (a) \cite{boulch2018snapnet, kalogerakis20173d, lawin2017deep, sinha2016deep, su2015multi} that project the 3D point cloud onto a set of virtual multi-view image planes, and (b) \cite{brock2016generative,dai2017scannet,maturana2015voxnet,tchapmi2017segcloud, wu20153d} that quantize the 3D point cloud into regular voxel grids. However, both approaches result in severe loss of information. The former suffers information loss from occlusions and 3D to 2D projections. The latter suffers from quantization error and the high computational cost of 3D convolutions limits its scalability. Recently, \cite{qi2017pointnet} pioneered the permutation invariant PointNet that allows deep learning to be directly applied to 3D point clouds. Furthermore, it has shown promising results on both 3D object classification and semantic segmentation tasks. Nonetheless, the design of PointNet, \ie feeding each point individually into several multi-layer perceptrons (MLPs) followed by a maxpooling operation to achieve permutation invariance, inherently prohibits the network from capturing the local information embedded in the neighboring points. However, the local information is essential in modeling fine-grained structures, \eg plane or corner, and convex or concave element. Consequently, several later works \cite{engelmann2017exploring, huang2018recurrent, landrieu2017large, li2018pointcnn, qi2017pointnet++, ye20183d, zeng20183dcontextnet} proposed to capture the local information by considering the neighboring points. However,
these approaches treat each point in their respective local region independently to achieve permutation invariance, and this prevents the modeling of geometric relationships among neighbor points, which hinders discriminative local feature learning. Moreover, those RNN-based approaches~\cite{engelmann2017exploring, huang2018recurrent, landrieu2017large, ye20183d} requiring sequential input ordering violate permutation invariance, and \cite{li2018pointcnn} inherently is not permutation-invariant.  DGCNN~\cite{wang2018dynamic} is designed to overcome such limitations. It embeds the relationship between a point and its neighbors in the so-called edge features, and uses a channel-wise symmetric aggregation operation (\ie max-pooling) over local neighborhood to capture representative local information. Nonetheless, the global context among different local features is neglected in their work due to a max pooling operation. This limits its ability to encode the semantic information of the entire 3D scene.

In this paper, we propose the PS$^2$-Net - an efficient end-to-end framework for \textbf{P}oint cloud \textbf{S}emantic \textbf{S}egmentation (PS$^2$) that takes local structures and global context into consideration. Our work leverages on EdgeConv~\cite{wang2018dynamic} to capture local information, and uses the NetVLAD~\cite{arandjelovic2016netvlad}  to encode global context. More specifically, PS$^2$-Net takes raw point cloud as the input and outputs point level semantic class labels. We design an encoder module that can be stacked repeatedly into a deep network for learning the discriminative representations. Each encoder module consists of two basic components - EdgeConv and NetVLAD. Furthermore, we prove theoretically that the PS$^2$-Net guarantees permutation invariance to any order of input points.  Our main contributions are summarized as follows:
\begin{itemize}
	\vspace{-0.1in}
	\item We design PS$^2$-net - an end-to-end network for semantic scene segmentation on 3D point clouds. Our network is permutation-invariant,  and is able to integrate both local structures and global context.
		\vspace{-0.1in}
	\item Our encoder is flexible and can be stacked or recurrently plugged into existing deep learning architectures to exploit the fine-grained local and global properties from point clouds.
		\vspace{-0.1in}
	\item Extensive experimental results on two large-scale 3D indoor datasets show that our  PS$^2$-net outperforms existing state-of-the-art approaches on the task of semantic scene segmentation on 3D point clouds.
\end{itemize}

\begin{figure*}[t]
	\centering
	\includegraphics[scale=0.7]{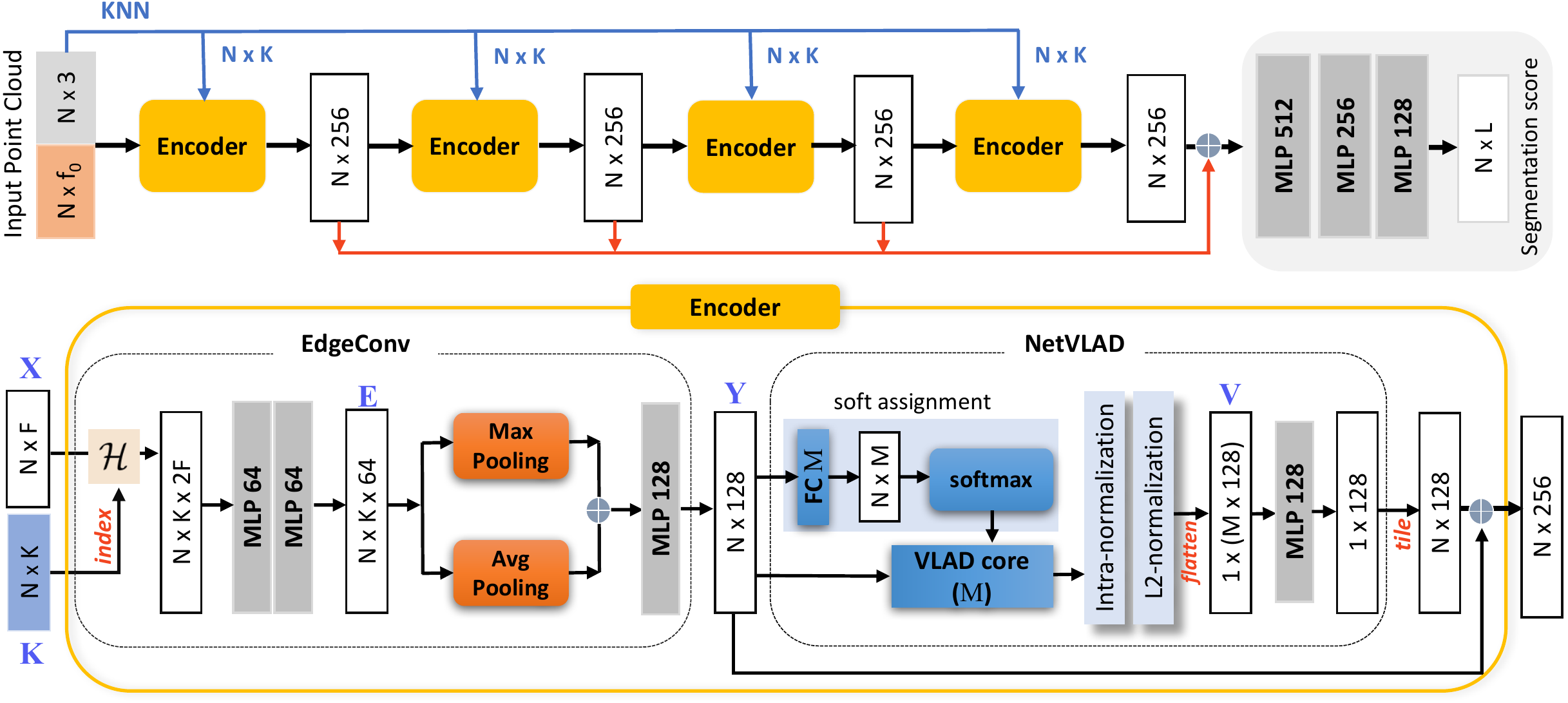}
	\caption{The network architecture of our PS$^2$-Net.}
	\label{fig:OurNetwork}
\vspace{-0.1in}
\end{figure*}

\section{Related work}
In this section, we focus on the literature survey of existing deep learning approaches that directly take point clouds as input for semantic scene segmentation. We omit discussions of existing works that alleviate the permutation invariance problem via conversion of the point clouds into alternative form of representations \eg projections into multi-view virtual images and voxelization.   

PointNet \cite{qi2017pointnet} pioneered the direct use of 3D points as input in deep learning.
It learns point-wise features by passing each point individually through several shared MLPs, followed by a symmetry function (\ie max-pooling) that aggregates over the feature channels into a global feature that represents the point cloud. The operation of the shared MLPs on the individual points and max-pooling made the network permutation invariant.
However, the design of PointNet overlooks the exploitation of local structures and prevents the network from learning fine-grained structures.
To overcome this limitation, most follow-up works attempt to incorporate the local structures in an efficient way. 
PointNet++~\cite{qi2017pointnet++} utilizes the farthest point sampling and ball query to group the point cloud into a set of local regions, and then applies PointNet on each region to capture the local structures. Despite the preservation of permutation invariance and the incorporation of local structures, the grouping operation on the point cloud into independent local regions may impede learning of the global context.

Other variants of PointNet employ various Recurrent Neural Network (RNN) based techniques to learn the contextual information from the local features, which are extracted by PointNet from partitioned local patches based on the spatial locations of the points~\cite{engelmann2017exploring, huang2018recurrent, ye20183d} or geometric homogeneousness~\cite{landrieu2017large}.
Engelmann~\etal~\cite{engelmann2017exploring} divide a point cloud into a tessellation of ordered blocks. A feature generated from each block using PointNet is then fed sequentially into a Gated Recurrent Unit (GRU). The GRU ensures that information from neighboring blocks are encoded in the final feature. 
Huang~\etal~\cite{huang2018recurrent} propose the RSNET, where the point cloud is sliced into three sets of uniformly-spaced blocks independently along the three coordinate axes (\ie x, y, z). A feature is extracted from each block using a PointNet-like network. Next, an ordered sequence for each set of blocks is defined manually along each axis. Each of the three ordered sequences is respectively fed into a RNN layer to explore the local context. Finally, the updated block features are assigned back to the points in the blocks. Thus, each point-wise feature aggregates its three associated directional block features. 
In contrast to an independent use of RNNs on the different axes as in RSNET, Ye~\etal~\cite{ye20183d} propose a two-direction hierarchical RNN model. Specifically, the first set of RNNs is applied on the blocks along the x-axis where the outputs are used by the second set of RNNs that runs along the y-axis to generate the block features across the two dimensions.
However, the RNN-based approaches require sequential ordering in the inputs and this is a violation of the permutation invariance.

More recently, several works~\cite{li2018pointcnn, shen2018mining, wang2018dynamic} are proposed to explore the fine-grained geometric relationships among neighboring points.
PointCNN~\cite{li2018pointcnn} learns a $K \times K$ $\chi$-transformation from the K neighborhoods of the input points. It assumes that each point assigned with a local coordinate can be permuted into a latent and potentially canonical order with the $\chi$-transformation. The learned $\chi$-transformation is further packed with typical convolutions to form a new process to extract features from local regions. However, the $\chi$-transformation is not permutation invariant.
KCNet~\cite{shen2018mining} presents a kernel correlation layer to exploit local geometric structures, where point-set kernel correlation is used to measure similarities between local neighborhood graphs and learned kernel graphs. Despite showing promising performance on object-level related tasks, it is somewhat sensitive to underlying graph structures and its capability in scene-level semantic segmentation is unknown.
DGCNN\cite{wang2018dynamic} explicitly designs the EdgeConv module to better capture local geometric features of point clouds while preserving permutation invariance. In particular, they dynamically compute a $k$-nearest neighbors (KNN) graph for each point in the feature spaces produced by different EdgeConv layers. EdgeConv applies multiple MLP layers over the combination of point-wise features and their respective KNN features to generate edge features, and outputs the enhanced point-wise features by pooling among neighboring edge features. 
However, the features learned from DGCNN does not encode contextual and semantic information of the whole 3D scene, i.e. global context.   
Furthermore, the dynamic re-computation of the KNN graphs during network training is extremely computationally expensive and inefficient when applied on large-scale point clouds. 

Our PS$^2$-Net leverages on EdgeConv to learn the local information, but uses static KNN graphs defined in the spatial coordinate space to reduce computational complexity. Additionally, we use the permutation invariant NetVLAD layer \cite{arandjelovic2016netvlad} to incorporate the global contextual information into our learned feature.

\section{Our PS$^2$-Net}

The network architecture of our PS$^2$-Net is presented in Figure~\ref{fig:OurNetwork}. 
Our network consists of four encoders that are repeatedly stacked together, and a final segment of shared 3-layer MLPs. The encoders output different levels of feature abstractions that are subsequently merged and passed to the shared MLPs for point level semantic label classification.  
The input to our network is an $N\times(3+f_0)$ matrix of $N$ points $\in \mathbb{R}^3$ and an additional feature $\in \mathbb{R}^{f_0}$ (\eg color or surface normal). The network output is an $N \times L$ matrix of predicted probabilities for the $L$ labels on each of the $N$ points. 
A KNN search is done on each of the $N$ points in the $\mathbb{R}^3$ space to produce an $N \times K$ matrix. This $N \times K$ matrix is fed into each of the encoder to group local regions in the corresponding feature spaces.

An encoder is made up of two components - EdgeConv and NetVLAD. The $N \times 128$ output of EdgeConv is fed into NetVLAD, and then concatenated with the $N \times 128$ output of NetVLAD to generate the final $N \times 256$ output of the encoder.
Each encoder first forms an $N\times K \times 2F$ tensor with the operator $\mathcal{H}$ (see next paragraph for more details) using the 
set of input features $\in \mathbb{R}^F$ that are the $K$-nearest neighbors $\in \mathbb{R}^3$.
This $N\times K \times 2F$ tensor is fed into several MLP layers to extract the $N \times K \times 64$ edge features that describe the relationships between each point and its neighbors. The feature dimension $F$ equals to $(3 + f_0)$ and $256$ for the first and subsequent encoders, respectively.
We then apply two symmetric operations, i.e. channel-wise max-pooling and average-pooling, to respectively transform the $K$ edge features into two locally aggregated representations. 
Subsequently, the two representations are concatenated and fed into another MLP layer to produce a set of point-wise features ($N \times 128$) as the output of EdgeConv.

More formally, let us denote the input feature matrix to EdgeConv as $\textbf{X}=\{\textbf{x}_1,...,\textbf{x}_N \mid \textbf{x}_i \in \mathbb{R}^{F}\}$, the indices of KNN as  $\textbf{K}=\{\textbf{k}_1,...,\textbf{k}_N \mid \textbf{k}_i \in \mathbb{R}^{K}\}$, the edge feature tensor as $\textbf{E}=\{\textbf{e}_1,...,\textbf{e}_N \mid \textbf{e}_i \in \mathbb{R}^{K \times 64}\}$, and output feature matrix of the EdgeConv as $\textbf{Y}=\{\textbf{y}_1,...,\textbf{y}_N \mid \textbf{y}_i \in \mathbb{R}^{128}\}$. Then, the $\mathcal{H}$ operator is given by:
\begin{equation}
	\mathcal{H}: \textbf{X} \longmapsto \textbf{H},
	\vspace{-0.1in}
\end{equation}
\noindent where 
\begin{subequations}
		\vspace{-0.1in}
	\begin{align}
		\textbf{H} &= \{\textbf{h}_1,...,\textbf{h}_N \mid \textbf{h}_i \in \mathbb{R}^{K \times 2F} \}, \\
		\textbf{h}_i(j) &= [\textbf{x}_i, \textbf{x}_j - \textbf{x}_i ],~~~~j \in \textbf{k}_i.
	\end{align} \label{eq:HOperator}
		\vspace{-0.15in}
\end{subequations}

\noindent $\textbf{h}_i(j)$ is the $j^{th}$ row of $\textbf{h}_i$ and $[\textbf{A},\textbf{B}]$ represents the concatenation of matrices $\textbf{A}$ and $\textbf{B}$. The edge features $\textbf{E}$ are obtained by: 
\begin{equation}
	\textbf{e}_i = \sigma(\mathcal{F}( \sigma(\mathcal{F}(\textbf{h}_i; \Theta_1)); \Theta_2 ) ),
	\label{eq:edgeFeatures}
		\vspace{-0.1in}
\end{equation}

\noindent where $\mathcal{F}$ represents a shared MLP layer, and $\sigma$ represents a non-linear ReLU activation function. $\Theta_1$ and $\Theta_2$ are two sets of learnable parameters in the two respective shared MLP layers. Finally, the output of edgeConv $\textbf{Y}$ is given by:
\begin{equation}\label{eq:EdgeConvOutput}
	\textbf{y}_i = \sigma(\mathcal{F}([\maxp_{k \in \textbf{k}_i}(\textbf{e}_{k}) , \avgp_{k \in \textbf{k}_i}(\textbf{e}_{k})];~\Theta_3)),
\end{equation}
where 
$\Theta_3$ is the set of learnable parameters for the EdgeConv output shared MLP.

By taking $\textbf{Y}$ as input, NetVLAD outputs an $1 \times (M \times 128)$ global descriptor, denoted as $\textbf{V}$. The $m^{th}$ 128-vector is given by:
\begin{equation}\label{eq:netvlad}
	\textbf{v}_m = \sum_{i=1}^N \frac{e^{\textbf{w}_m^T \textbf{y}_i + b_m}}{\sum_{{m'}=1}^M e^{\textbf{w}_{m'}^T \textbf{y}_i + b_{m'}}} (\textbf{y}_i - \textbf{c}_m).
\end{equation}
There are two sets of parameters in the NetVLAD module: $M$ cluster centers (``visual words'') $\{\textbf{c}_1, ..., \textbf{c}_M \mid \textbf{c}_m \in \mathbb{R}^{128}\}$ and $\{\textbf{w}_m, b_m\}$ for learning the soft assignments, which determines the propagation of information from the input point-wise feature vectors to cluster centroids. 
Specifically, NetVLAD uses a MLP with $M$ feature channels followed by a soft-max function to obtain the soft-assignments. Next, the residuals between input point-wise descriptors and $M$ cluster centroids are aggregated with the soft-assignments. 
Finally, the matrix is flatten into an $1 \times (M \times 128)$ output vector $\textbf{V}$ with intra- and inter- normalization. 
To be computational efficient, this high dimensional vector is further compressed into a concise feature representation via a MLP, which is treated as the final global descriptor.

\subsection{Exploitation of Local Structure}
As mentioned earlier, we take inspiration from EdgeConv in \cite{wang2018dynamic} to exploit the local structures among the local neighborhood of a point. We replace the original dynamic KNN graphs computation with static KNN graphs computed on metric space for two reasons. First, we assume that the static metric-based KNN graphs may supervise local structure learning with spatial constraint. Second, static KNN graphs is analogous to the structure of images, where the neighborhoods of pixels remain fixed during convolution.    
Additionally, the two symmetric operations in aggregating edge features, i.e. channel-wise max-pooling and average-pooling, are designed to compensate the information loss during aggregation. Max-pooling operator individually performs over $K$ dimensions and selects the maximum feature responses over the neighborhood of each point, while average pooling operator sums up all the features in the same local region. By combining these two operations, fine-grained local information are preserved and the strong responses are emphasized.   
We progressively expand the receptive field by repeatedly stacking the encoders.

\subsection{Aggregation of Global Context}
In addition to exploiting the local structures, we also design our network to be aware of the global context that provides scene-level semantic information. This information is potentially useful and can alleviate local confusions. Particularly, semantic context helps to distinguish patches with similar appearances/geometrics but different semantic meanings (\eg differentiate a white wall from a white board on it). To incorporate this globally contextual information, we leverage on the success of NetVLAD \cite{arandjelovic2016netvlad}, which is a technique originally designed for aggregating local descriptors into a global vector in the image domain. \cite{uy2018pointnetvlad} adopted it into
PointNetVLAD that generates global descriptors for point-based inputs. The descriptor vector of each cluster center is a summation of residuals (\ie contributions) of each input feature to the center. Consequently, it is able to reveal fine-grained global contexts due to the large receptive field and aggregation of the learned relationships with all points.

\subsection{Permutation Invariance}
\noindent We now prove that the PS$^2$-Net is permutation-invariant. 
\begin{lemma}\label{lemma1}
	The PS$^2$-Net is permutation invariant, \ie if the rows of the input point cloud matrix are permuted, the output of the network remains unchanged.
\end{lemma}
\begin{proof}
	Let $\textbf{P}=\{\textbf{p}_1,\textbf{p}_2,..,\textbf{p}_N\}$ denote the input $N \times 3$ matrix, and $\textbf{L}$ denote the output $N \times L$ matrix. Now we want to prove if the input is $\pi\textbf{P}$ where $\pi$ is an $N \times N$ permutation matrix, while the output of PS$^2$-Net remains as $\textbf{L}$.
	As the shared MLPs operating on individual points are obviously permutation-invariant. Here we simplify the proof by proving the permutation invariance property of the encoder in PS$^2$-Net. 
	Suppose we have a permuted point cloud 
	\begin{equation*}
		\Tilde{\textbf{P}}=\{\textbf{p}_1,...,\textbf{p}_{i-1}, \textbf{p}_{j},\textbf{p}_{i+1},...,\textbf{p}_{j-1},\textbf{p}_{i},\textbf{p}_{j+1},...,\textbf{p}_N\}
	\end{equation*}
	that only reorders points $\textbf{p}_{i}$ and $\textbf{p}_{j}$ in $\textbf{P}$, then the feature representation of
	$\Tilde{\textbf{P}}$ is given by 
	\begin{equation*}
		\Tilde{\textbf{X}}=\{\textbf{x}_1,...,\textbf{x}_{i-1}, \textbf{x}_{j},\textbf{x}_{i+1},...,\textbf{x}_{j-1},\textbf{x}_{i},\textbf{x}_{j+1},...,\textbf{x}_N\}.
	\end{equation*}
	As the reordering does not affect the order of nearest neighbors, the KNN indices are still given by \textbf{K}. Inputting $\Tilde{\textbf{X}}$ and \textbf{K} into the encoder, the output is formulated as:
	\begin{equation}\label{eq:NetVLAD}
		\Tilde{\textbf{V}} = f^{\nu}(f^{e}(\Tilde{\textbf{X}},\textbf{K})),
	\end{equation}
	where $f^{e}(\cdot)$ denotes a series of operations in EdgeConv $\Tilde{\textbf{Y}} = f^e(\Tilde{\textbf{X}},\textbf{K})$
	and $f^{\nu}(\cdot)$ denotes the NetVLAD function $\Tilde{\textbf{V}} = f^{\nu}(\Tilde{\textbf{Y}})$.
	
	Putting the $\Tilde{\textbf{X}}$ into the operator $\mathcal{H}$ in Equation \ref{eq:HOperator}, we get
	\begin{equation*}
		\Tilde{\textbf{H}}=\{\textbf{h}_1,...,\textbf{h}_{i-1},\textbf{h}_{j},\textbf{h}_{i+1},...,\textbf{h}_{j-1},\textbf{h}_{i},\textbf{h}_{j+1},...,\textbf{h}_N\}.    
	\end{equation*}
	
	\noindent Since each $\textbf{e}_i$ is processed independently in the shared MLPs (see Equation \ref{eq:edgeFeatures}), it is obvious that
	\begin{equation*}
		\Tilde{\textbf{E}}=\{\textbf{e}_1,...,\textbf{e}_{i-1},\textbf{e}_{j},\textbf{e}_{i+1},...,\textbf{e}_{j-1},\textbf{e}_{i},\textbf{e}_{j+1},...,\textbf{e}_N\}. 
	\end{equation*}
	
	\noindent Finally, the output of edgeConv is given by
	\begin{equation*}
		\Tilde{\textbf{Y}}=\{\textbf{y}_1,...,\textbf{y}_{i-1},\textbf{y}_{j},\textbf{y}_{i+1},...,\textbf{y}_{j-1},\textbf{y}_{i},\textbf{y}_{j+1},...,\textbf{y}_N\},
	\end{equation*}
	
	\noindent again this is because each $\textbf{e}_i$ is processed independently in the shared edgeConv output MLP in Equation \ref{eq:EdgeConvOutput}.

	The output of NetVLAD $\textbf{V}=\{\textbf{v}_1,\textbf{v}_2,...,\textbf{v}_M\}$ of the original input features $\textbf{X}$ before permutation can now be written as
	\begin{equation}
		\begin{split}
			\textbf{v}_m &= f^\nu_m(\textbf{y}_1) + f^\nu_m(\textbf{y}_2) +... +f^\nu_m(\textbf{y}_N) \\
			&= \sum_{s=1}^N f^\nu_m(\textbf{y}_s),~~~~~\forall m,
		\end{split}
	\end{equation}
	where 
	\begin{equation}
		f^\nu_m(\textbf{y}_s) = \frac{e^{\textbf{w}_m^T \textbf{y}_s + b_m}}{\sum_{{m'}=1}^M e^{\textbf{w}_{m'}^T \textbf{y}_s + b_{m'}}} (\textbf{y}_s - \textbf{c}_m);
	\end{equation}
	and the output of NetVLAD with the permuted input features is given by
	\begin{equation}
		\begin{split}
			\Tilde{\textbf{v}}_m &= f^\nu_m(\textbf{y}_1) + ... + f^\nu_m(\textbf{y}_{i-1}) + f^\nu_m(\textbf{y}_{j}) + \\ &~~~~~f^\nu_m(\textbf{y}_{i+1}) + ... + f^\nu_m(\textbf{y}_{j-1}) + f^\nu_m(\textbf{y}_{i}) + \\ &~~~~~f^\nu_m(\textbf{y}_{j+1}) + ... + f^\nu_m(\textbf{y}_N) \\
			&= \sum_{s=1}^N f^\nu_m(\textbf{y}_s) = \textbf{v}_m,~~~~~\forall m. 
		\end{split}
	\end{equation}
	
	\noindent Hence, 
	\begin{equation}
		\textbf{V} = f^{\nu}(f^{e}(\Tilde{\textbf{X}},\textbf{K})).
	\end{equation}
	This completes our proof that the encoder of PS$^2$-Net is permutation invariant.
	
\end{proof}

\section{Experiments}
\subsection{Datasets}
We conduct experiments of our network on two challenging benchmark datasets
: S3DIS~\cite{armeni20163d} and ScanNet~\cite{dai2017scannet}. The details of these two datasets are as follows.

\vspace{-0.15in}
\paragraph{S3DIS} This dataset consists of 6 different indoor areas from 3 different buildings. It has 271 rooms with various styles, \eg conference room, lobby, restroom. 
Each point is labeled by one of 13 semantic classes. These classes are partitioned into \textit{structural type} (ceiling, floor, wall, beam, column, window, door), \textit{furniture type} (table, chair, sofa, bookcase, board) and \textit{clutter}. In the experiments, we adopt the 6-fold training/testing split used in~\cite{qi2017pointnet}.

\vspace{-0.15in}
\paragraph{ScanNet} This dataset consists of 1,513 scans from 707 unique indoor scenes. The space type is very diverse, ranging from very small (\eg bathroom, closet, utility room) to very large (\eg apartment, classroom, library) spaces. 
Each point is annotated by one out of 21 semantic classes, including 20 object classes plus 1 extra class representing free space.
Following the experimental settings in~\cite{qi2017pointnet++}, we adopt 1,201 scans for training and 312 scans for testing.

\subsection{Implementation Details}
Our PS$^2$-Net consists of four repeatedly stacked encoders with the same configuration. Each EdgeConv module has two shared MLPs (64,64) to extract edge features and one shared MLP (128) to fuse the max- and avg- pooled edge features. 
The NetVLAD module has 16 clusters and produces a $(16 \times 128)$-dimensional global descriptor that is fed into a shared MLP (128) for dimension reduction. 
Skip link is is added from the output of EdgeConv to the output of NetVLAD for integrating local and global features. 
The outputs from all the four encoders are concatenated and forwarded to three shared MLPs (512,256,128) to map the learned features to point labels. Dropout with a drop-ratio of 0.3 is used in the first layer.
Batch normalization and ReLU are added to all the respective MLPs.
The number of nearest neighbors is set to 20.

\begin{table*}[t]
	\centering
	\caption{Comparison of performances on S3DIS dataset using 6-fold cross validation. Results in the upper and lower table are obtained with data preparation setups in~\cite{qi2017pointnet} (denoted as \textit{P1}) and~\cite{li2018pointcnn} (denoted as \textit{P2}), respectively. Class-wise IoU is also given.}
	\scalebox{0.76}{
		\begin{tabular}{l||l l||l l l l l l l l l l l l l}
			\hline\toprule[1pt]
			\textbf{Method}  & \textbf{OA}  & \textbf{mIoU} & \textbf{ceiling} & \textbf{floor} &\textbf{wall} & \textbf{beam}& \textbf{column}& \textbf{window}& \textbf{door}& \textbf{table}& \textbf{chair}& \textbf{sofa}  &  \textbf{bookcase}&\textbf{board} & \textbf{clutter} \\\hline
			PointNet~\cite{qi2017pointnet} & 78.5 & 47.6 & 88  & 88.7  & 69.3 & 42.4  & 23.1  & 47.5 & 51.6  & 54.1 & 42  & 9.6 & 38.2  &29.4 & 35.2 \\\hline
			G+RCU~\cite{engelmann2017exploring} & 81.1 & 49.7 & 90.3  & 92.1 & 67.9 & \textbf{44.7} & 24.2 & 52.3 & 51.2 & 58.1 & 47.4 & 6.9 & 39 & 30 & 41.9 \\\hline
			DGCNN~\cite{wang2018dynamic} & 84.1 & 56.1 & --- & --- & --- & --- & --- & --- & --- & --- & --- & --- & --- & --- & --- \\\hline
			RNNCF~\cite{ye20183d} & \textbf{86.9} & 56.3 & 92.9 & 93.8 & 73.1 & 42.5 & 25.9 & 47.6 & 59.2 & 60.4 & \textbf{66.7} & 24.8 & \textbf{57} & 36.7 & 51.6 \\\hline
			RSNet~\cite{huang2018recurrent}  & --- & 56.47 & 92.48 & 92.83 & 78.56 & 32.75 & 34.37 & 51.62 & \textbf{68.11} & 60.13 & 59.72 & \textbf{50.22} & 16.42 & 44.85 & 52.03 \\\hline
			\textbf{PS$^2$-Net---P1} & 86.69 & \textbf{61.56} & \textbf{93.40} & \textbf{95.64} & \textbf{79.94} & 37.17 & \textbf{40.93} & \textbf{59.83} & 66.65 & \textbf{63.65} & 65.71 & 37.16 & 49.83 & \textbf{54.56} & \textbf{55.66} \\\hline \hline
			PointCNN~\cite{li2018pointcnn}  & 88.14 & 65.39 & \textbf{94.78} & \textbf{97.3} & 75.82 & \textbf{63.25} & 51.71 & 58.38 & 57.18 & \textbf{71.63} & 69.12 & 39.08 & \textbf{61.15} & 52.19 & 58.59  \\\hline
			\textbf{PS$^2$-Net---P2} & \textbf{88.22} & \textbf{66.60} & 93.04 & 96.26  & \textbf{83.22} & 41.61 & \textbf{54.05} & \textbf{60.08} & \textbf{70.40} & 67.37 & \textbf{73.13} & \textbf{48.75} & 58.73 & \textbf{58.68} & \textbf{60.48} \\\hline\toprule[1pt]
	\end{tabular}}
	\label{tbl:s3dis-6fold}
	\vspace{-0.1in}
\end{table*}

Our framework is implemented using PyTorch deep learning library on a NVIDIA GTX 1080Ti.
We optimize the network using ADAM~\cite{kingma2014adam} with an initial learning rate of 0.001 and weight decay 1e-5. Multi-class cross-entropy is used as the loss function. In most experiments, the learning rate is decayed by half after every 100 epochs. In general, the networks converged at $\sim$150 epochs. 
The batch size is set to 6 for experiments with 4,096 points as input, otherwise, the batch size is updated with respect to the number of input points. Note that we do not use any data augmentation in our experiments.

We adopt the two widely used metrics: overall accuracy (OA) and mean interaction over union (mIoU)
to evaluate the segmentation performance of our network. Additionally, we also report the class-wise IoU, which is computed for each point that belongs to its corresponding semantic class.

\subsection{Segmentation on S3DIS Dataset}
\subsubsection{Data preparation}
We observed some differences in the data preparation setups among the existing approaches for S3DIS dataset. 
One widely-adopted setup is proposed by PointNet~\cite{qi2017pointnet}, which splits each room into non-overlapping blocks 
of 1m$\times$1m area on the $xy$ plane and each point is represented by a 9-dim vector containing the $(x,y,z)$ coordinates, $(r,g,b)$ color and normalized coordinates. 
4,096 points are randomly sampled from each block during training and testing.
The other setup is presented by PointCNN~\cite{li2018pointcnn}. It slices the rooms into 1.5m-by-1.5m blocks with 0.3m padding on each side and each point is associated with a 6-dim vector containing the $(x,y,z)$ coordinates and $(r,g,b)$ color. $\mathcal{N}(2048, 256^2)$ ($\mathcal{N}$ denotes the Gaussian distribution) points are sampled from each block during training, while each block is sampled multiple times to make sure that all the points are evaluated during testing. 
In order to make fair comparisons, we conduct experiments using both data preparation setups\footnote{We use the codes in \hyperlink{https://github.com/charlesq34/pointnet/tree/master/sem_seg}{PointNet} and \hyperlink{https://github.com/yangyanli/PointCNN/tree/master/data_conversions}{PointCNN} for preprocessing.} and report the comparisons accordingly.

\begin{figure}[h]
	\centering
	\includegraphics[scale=0.55]{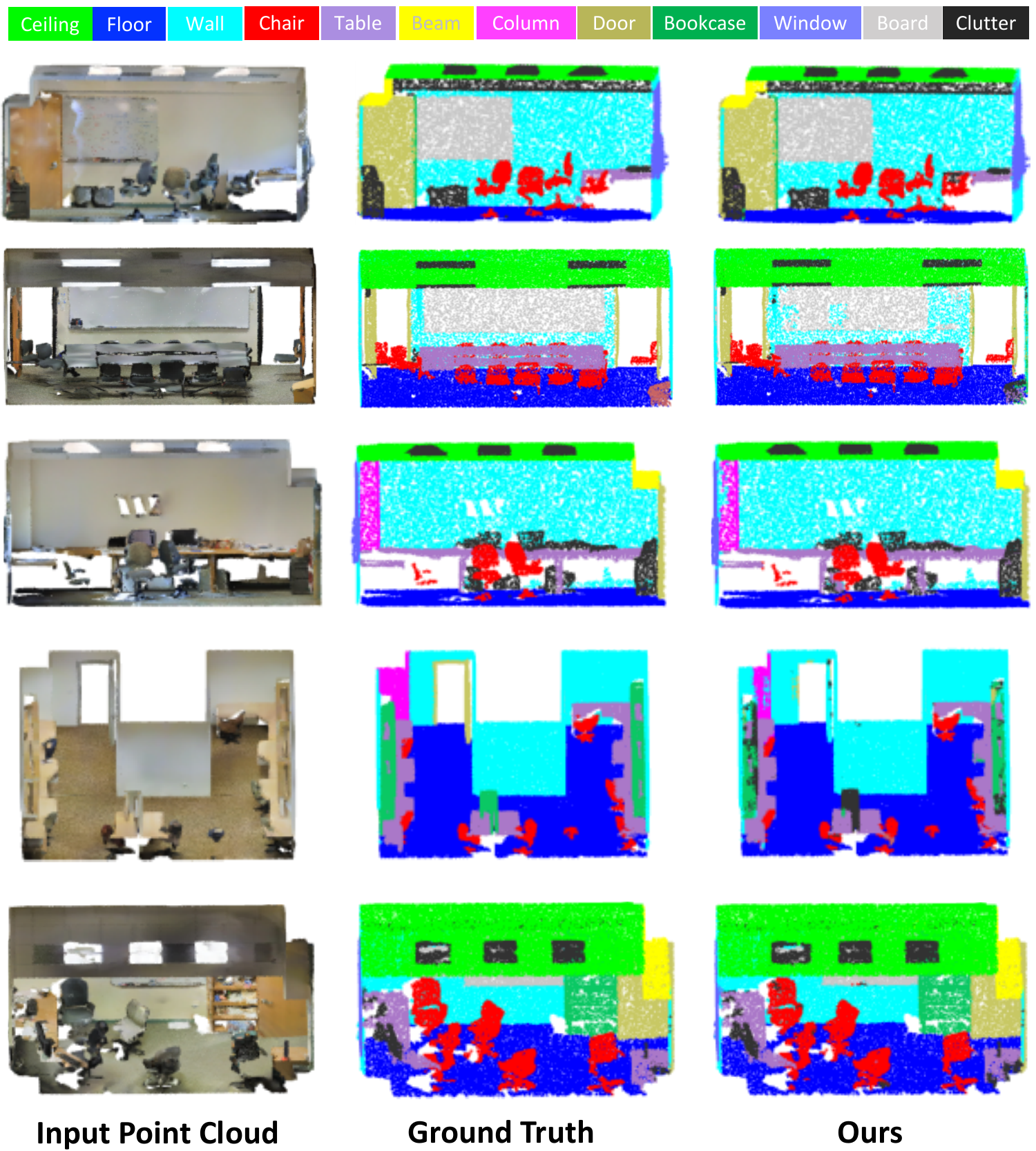} 
	\caption{PS$^2$-Net semantic segmentation results on S3DIS.}
	\label{fig:s3dis-visualization}
	\vspace{-0.1in}
\end{figure}

\begin{table*}[t]
	\centering
	\caption{Comparison of performances on ScanNet dataset using XYZ information as input. Results in the upper and lower table are obtained with data preparation setups in~\cite{qi2017pointnet++} (denoted as \textit{P3}) and~\cite{li2018pointcnn} (denoted as \textit{P2}), respectively. Class-wise IoU is also given.}
	\scalebox{0.78}{
		\begin{tabular}{p{2.5cm}||p{1cm} p{1cm}||p{1.17cm} p{1.17cm} p{1cm} p{1cm} p{1cm} p{1cm} p{1.6cm} p{1cm} p{1cm} p{1.5cm} }
			\hline\toprule[1pt]
			\textbf{Method}  & \textbf{OA}  & \textbf{mIoU} & \textbf{wall} & \textbf{floor} &\textbf{chair} & \textbf{table}& \textbf{desk}& \textbf{bed}& \textbf{bookshelf}& \textbf{sofa}& \textbf{sink}  &\textbf{bathtub}  \\\hline
			PointNet~\cite{qi2017pointnet} & --- & 14.69 & 69.44 & 88.59 & 35.93 & 32.78 & 2.63 & 17.96 & 3.18 & 32.79 & 0 & 0.17\\\hline
			PointNet++~\cite{qi2017pointnet++}& --- & 34.26 & 77.48 & 92.5 & 64.55 & 46.6 & 12.69 & 51.32 & 52.93 & 52.27 & 30.23 & 42.72\\\hline
			RSNet~\cite{huang2018recurrent} & --- & 39.35 & \textbf{79.23} & \textbf{94.1} & 64.99 & \textbf{51.04} & \textbf{34.53} & \textbf{55.95} & \textbf{53.02} & \textbf{55.41} & \textbf{34.84} & 49.38\\\hline
			\textbf{PS$^2$-Net---P3} & --- & \textbf{40.17}  & 72.44 & 91.51 & \textbf{65.08} & 45.61 & 26.27 & 48.90 & 39.96 & 53.94 & 24.58 & \textbf{64.78} \\\hline\hline
			PointCNN~\cite{li2018pointcnn}& 85.1 & --- & --- & --- & --- & --- & --- & --- & --- & --- & --- & ---\\\hline
			\textbf{PS$^2$-Net---P2} & \textbf{87.21} & 44.90  & 77.02 & 91.22 & 68.36 & 56.66 & 31.62 & 53.55 & 36.32 & 58.75 & 43.07 & 70.11\\\hline\toprule[1pt]
	\end{tabular}}
	\newline
	\vspace*{0.1 cm}
	\newline
	\scalebox{0.78}{
		\begin{tabular}{p{2.5cm}||p{1cm} p{1.2cm} p{1.2cm} p{0.8cm} p{1.2cm} p{2.4cm} p{2cm} p{1cm} p{1.1cm} p{2.5cm}}
			\hline\toprule[1pt]
			\textbf{Method}  & \textbf{toilet} & \textbf{curtain} & \textbf{counter} & \textbf{door} & \textbf{window} & \textbf{shower curtain} & \textbf{refridgerator} & \textbf{picture} & \textbf{cabinet} & \textbf{other furniture} \\\hline
			PointNet~\cite{qi2017pointnet} & 0 & 0 & 5.09 & 0 & 0 & 0 & 0 & 0 & 4.99 & 0.13 \\\hline
			PointNet++~\cite{qi2017pointnet++} & 31.37 & 32.97 & 20.04 & 2.02 & 3.56 & 27.43 & 18.51 & 0  & 23.81 & 2.2\\\hline
			RSNet~\cite{huang2018recurrent} & 54.16 & 6.78 & 22.72 & 3 & 8.75 & 29.92 & \textbf{37.9} & 0.95 & \textbf{31.29} & \textbf{18.98}\\\hline
			\textbf{PS$^2$-Net---P3} & \textbf{60.87} & \textbf{41.20} & \textbf{24.62} & \textbf{8.37} & \textbf{21.55} & \textbf{47.24} & 19.68 & \textbf{2.63} & 28.02 & 16.22 \\\hline\hline
			PointCNN~\cite{li2018pointcnn}& --- & --- & --- & --- & --- & --- & --- & ---  & --- & ---\\\hline
			
			\textbf{PS$^2$-Net---P2} & 66.28 & 41.94 & 23.73 & 10.94 & 17.82 & 51.02 & 44.19 & 3.17 & 32.82 & 19.33   \\\hline\toprule[1pt]
	\end{tabular}}
	\label{tbl:scannet-comparison}
	\vspace{-0.1in}
\end{table*}

\subsubsection{Results and Discussion}
Table~\ref{tbl:s3dis-6fold} illustrates the performance of our PS$^2$-Net in comparison to previous state-of-the-art approaches on the S3DIS dataset. All approaches~\cite{engelmann2017exploring, huang2018recurrent, wang2018dynamic, ye20183d} from the upper table prepared their experimental data according to the data preparation setup in PointNet~\cite{qi2017pointnet}. We can see that our PS$^2$-Net achieves the best performance in the mIoU criteria, which is a more precise criteria than the overall accuracy as the dataset is highly unbalanced. In particular, our PS$^2$-Net improves the mIoU by 9.7\% and overall accuracy by 3\% when compared with DGCNN~\cite{wang2018dynamic}, which is made up of a stacking of EdgeConv layers. We argue that the incorporation of global context in our network contributes to this significant improvement. Additionally, our PS$^2$-Net outperforms the three existing RNN-based methods~\cite{engelmann2017exploring, huang2018recurrent,ye20183d} in the mIoU criteria. 

In the lower part of Table~\ref{tbl:s3dis-6fold}, we show competitive performance of our PS$^2$-Net with PointCNN~\cite{li2018pointcnn} when the same data preprocessing setup is used. Moreover, the performance of our PS$^2$-Net---P2 increases by 8.2\% in mIoU compared to PS$^2$-Net---P1. This indicates that data preprocessing plays an important role in training, and the extensive online sampling used in PointCNN is able to expose the model to larger groups of samples during training.

Furthermore, several qualitative results on the S3DIS dataset are shown in Figure~\ref{fig:s3dis-visualization}. As we can see from the examples, the dataset is very challenging in many scenarios, \eg ``the white boards on the while wall'', ``the open doors with only visible door frames'', and ``the white column in the boundaries of the white wall'' \etc. Interestingly, our PS$^2$-Net successfully segments the boards, doors, columns in most cases despite the similar colors (\eg wall vs column) or geometries (\eg wall vs board). We believe that the correct classifications are consequences of integrating local structures and global context into our network. 

\begin{figure}[h]
	\centering
	\includegraphics[scale=0.5]{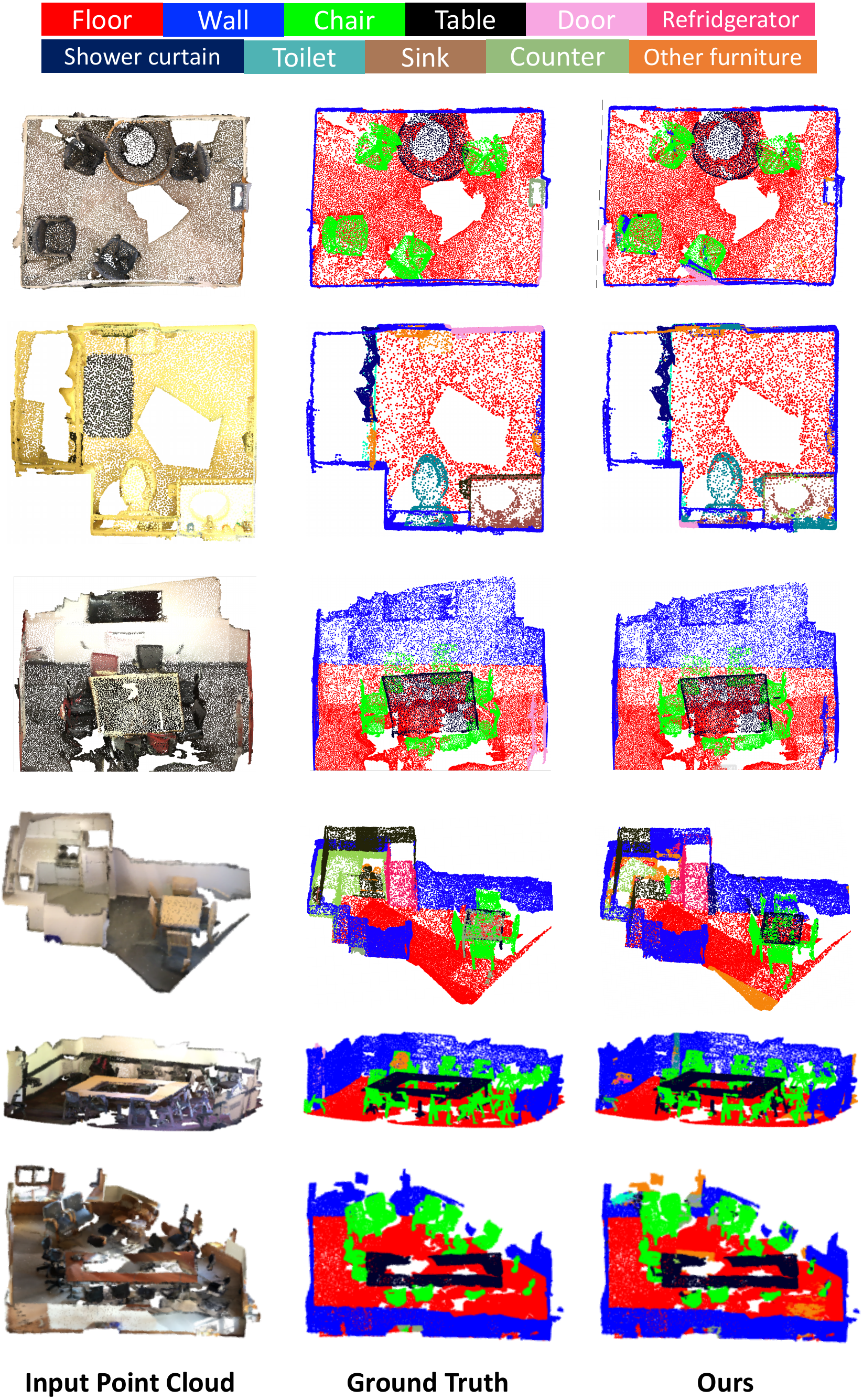}
	\caption{PS$^2$-Net semantic segmentation results on ScanNet.}
	\label{fig:scannet-visualization}
	\vspace{-0.2in}
\end{figure}

\subsection{Segmentation on ScanNet Dataset}
\subsubsection{Data preparation}
Similar to S3DIS, there are two data preparation setups in the ScanNet dataset. The first setup is proposed by PointNet++~\cite{qi2017pointnet++}. It follows the same setup in \cite{dai2017scannet} to first generate a tessellation of 1.5m$\times$1.5m$\times$3m cubes with 2cm$^3$ voxels from the 3D point clouds. Next, cubes 
with $\geq2\%$ occupied voxels and $\geq70\%$ valid annotations on the voxel surfaces are extracted. 

During training, 8,192 points are sampled from each cube. Each point is represented by only the $(x,y,z)$ coordinates. The second setup is proposed by PointCNN~\cite{li2018pointcnn}. It prepares the data in the same way as S3DIS except that only the $(x,y,z)$ coordinates are used on each point. 
Additionally, it converts the the segmentation results on the test data into semantic voxel labeling for comparison with test results from the first data preprocessing setup. 
Again, to make fair comparisons, we conduct experiments using both data processing setups\footnote{We use the code of \hyperlink{https://github.com/charlesq34/pointnet2/tree/master/scannet}{PointNet++} for preprocessing.}.
Note that different from S3DIS, we only use the $(x,y,z)$ information as the input in this dataset in order to be compliant with the previous approaches.

\subsubsection{Results and Discussion}
The comparison of performances on the ScanNet dataset is summarized in Table~\ref{tbl:scannet-comparison}. 
It can be seen that the performance of our PS$^2$-Net outperforms the previous state-of-the-art methods~\cite{huang2018recurrent, qi2017pointnet++} in the mIoU criteria when the prepocessing setup in PointNet++~\cite{qi2017pointnet++} is used.
Notably, we achieve remarkable improvements in the classification results of several challenging classes with very few training data, \ie bathtub (0.3\%), toilet (0.3\%), curtain (1.5\%), window (0.9\%) and shower curtain (0.2\%). We reckon that the increase in performance comes from the use of EdgeConv, where the discriminative representations learned from local structures are superior to the point-wise features that are individually processed in \cite{huang2018recurrent, qi2017pointnet, qi2017pointnet++}.
We obtain an impressive improvement in mIoU (11.8\%) on PS$^2$-Net---P2 compared to PS$^2$-Net---P3; and furthermore, we surpass PointCNN~\cite{li2018pointcnn}. 
Particularly, our PS$^2$-Net achieves more accurate segmentation results (class-wise IoU $>$0.4) in 10 out of 20 classes.

Several segmentation results are visualized in Figure~\ref{fig:scannet-visualization}. Our PS$^2$-Net is able to recognize both frequently, \eg wall, floor, chair, and rarely, \eg toilet, sink, refridgerator, seen objects. 

\subsection{Ablation studies}
In this section, we investigate the contributions of the respective components (\ie EdgeConv and NetVLAD) in our PS$^2$-Net, and evaluate the effects of several key hyper-parameters (\ie number of  nearest neighbors $K$ in EdgeConv, clusters $M$ in NetVLAD, and stacked encoders). We conduct ablation experiments on the fifth fold of the S3DIS dataset, \ie we test on Area 5 and train on the remaining data. More specifically, the testing area is collected in a separated building from all the training areas.
All settings remain unchanged as the baseline PS$^2$-Net in the ablation experiments except the target hyper-parameter.

\begin{table}[t]
	\centering
	\caption{Ablation test of PS$^2$-Net variants on S3DIS [A5].}
	\begin{tabular}{p{3cm}|p{1.5cm}|p{1.5cm}}
		\hline\toprule[0.5pt]
		Model & OA & mIoU \\ \hline
		w/o local features & 81.66 & 45.25 \\\hline
		w/o EdgeConv & 83.50 & 50.35\\\hline
		w/o NetVLAD & 84.00 & 50.18 \\\hline
		Our full PS$^2$-Net & \textbf{84.60} & \textbf{52.95}\\\hline\toprule[0.5pt]
	\end{tabular}
	\label{tab:ablation-modules}
\end{table}

\paragraph{Effectiveness of network modules}  We study the effect of each network module 
by removing them individually from the network, and compare the performances before and after the removal. To further measure the contribution of local features, we design two settings: in (1) ``w/o local features'', we substitute the EdgeConv module with a point-wise feature learning module which does not exploit local structures, and in (2) ``w/o EdgeConv'', we apply point-wise feature learning followed by a max-pooling operator over K-nearest neighbors to extract the local features. Additionally, in ``w/o NetVLAD'' we replace NetVLAD with a max-pooling operation similar to DGCNN~\cite{wang2018dynamic}.
Table~\ref{tab:ablation-modules} shows the comparison results. The poor performance of ``w/o local features'' in mIoU accords well with our claim of the importance of local information in modeling fine-grained structures. Furthermore, we can see that the integration of local structures and global context in our PS$^2$-Net contribute to the improvements over ``w/o EdgeConv'' and ``w/o NetVLAD''.

\begin{figure}[t]
	\centering
	\includegraphics[scale=0.3]{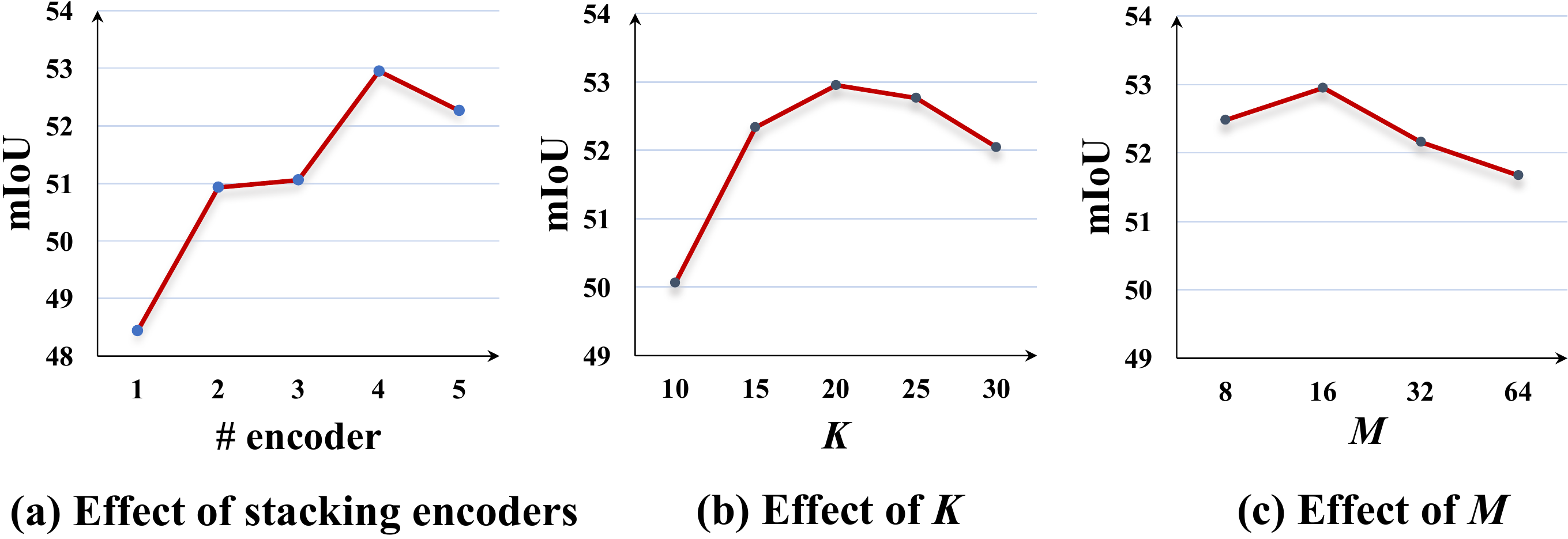}
	\caption{Effectiveness of hyper-parameters on S3DIS [A5].}
	\label{fig:hyperparamter}
\end{figure}

\paragraph{Number of encoders}  We evaluate the performance of our PS$^2$-Net with different number of stacked encoders, and show the results in Figure~\ref{fig:hyperparamter}(a). Generally, the network becomes deeper and is capable of learning more discriminative representations when the number of encoders increases. This can be seen from the improvement of performance when the number of encoders is increased from 1 to 4. However, we also observe that the performance begins to drop after a certain number of encoders. This might be due to a deeper network with more parameters requires more training data to prevent overfitting. In other words, a deeper network may fail to generalize well to new cases
when the training dataset is not sufficiently large. Hence, we use the best setting of 4 encoders in all our experiments.

\paragraph{Number of nearest neighbors} This hyper-parameter controls the range of local regions and thus influences the amount of local information included in our network. From Figure~\ref{fig:hyperparamter}(b), we can see a trade-off in selecting $K$. A small $K$ may lead to limited regions with insufficient local context, while a big $K$ may bring irrelevant noises and increases the computation complexity. 
We set $K$ to 20 in the experiments according to Figure~\ref{fig:hyperparamter}(b).

\paragraph{Number of clusters} 
As is revealed in Figure~\ref{fig:hyperparamter}(c), this hyper-parameter gives lower performances when $M$ is bigger than 16. It is likely because we reduce the dimension of $V$ from $(M \times 128)$-dim to 128-dim. If $M$ is too big, the dimension reduction operation may induce uncertain information loss. Consequently, we selected $M$=16 that achieved the best performance.

\section{Conclusion}
In this paper, we proposed the PS$^2$-Net, an end-to-end deep neural network for the 3D point cloud semantic segmentation task. PS$^2$-Net is built on four repeatedly stacked encoders, where each encoder has two basic components: EdgeConv and NetVLAD that capture local structures and global context, respectively. We provided proof to guarantee the permutation invariance property of our PS$^2$-Net. 
We obtained state-of-the-art performances with our locally and globally aware PS$^2$-Net on the two challenging 3D indoor scene datasets for point-based semantic segmentation.

{\small
\bibliographystyle{ieee}
\bibliography{reference}

\begin{thebibliography}{10}\itemsep=-1pt

\bibitem{arandjelovic2016netvlad}
R.~Arandjelovic, P.~Gronat, A.~Torii, T.~Pajdla, and J.~Sivic.
\newblock Netvlad: Cnn architecture for weakly supervised place recognition.
\newblock In {\em CVPR}, pages 5297--5307, 2016.

\bibitem{armeni20163d}
I.~Armeni, O.~Sener, A.~R. Zamir, H.~Jiang, I.~Brilakis, M.~Fischer, and
  S.~Savarese.
\newblock 3d semantic parsing of large-scale indoor spaces.
\newblock In {\em CVPR}, pages 1534--1543, 2016.

\bibitem{badrinarayanan2017segnet}
V.~Badrinarayanan, A.~Kendall, and R.~Cipolla.
\newblock Segnet: A deep convolutional encoder-decoder architecture for image
  segmentation.
\newblock {\em IEEE Transactions on Pattern Analysis \& Machine Intelligence},
  (12):2481--2495, 2017.

\bibitem{boulch2018snapnet}
A.~Boulch, J.~Guerry, B.~Le~Saux, and N.~Audebert.
\newblock Snapnet: 3d point cloud semantic labeling with 2d deep segmentation
  networks.
\newblock {\em Computers \& Graphics}, 71:189--198, 2018.

\bibitem{brock2016generative}
A.~Brock, T.~Lim, J.~M. Ritchie, and N.~Weston.
\newblock Generative and discriminative voxel modeling with convolutional
  neural networks.
\newblock {\em arXiv preprint arXiv:1608.04236}, 2016.

\bibitem{chen2018deeplab}
L.-C. Chen, G.~Papandreou, I.~Kokkinos, K.~Murphy, and A.~L. Yuille.
\newblock Deeplab: Semantic image segmentation with deep convolutional nets,
  atrous convolution, and fully connected crfs.
\newblock {\em IEEE Transactions on Pattern Analysis \& Machine Intelligence},
  40(4):834--848, 2018.

\bibitem{dai2017scannet}
A.~Dai, A.~X. Chang, M.~Savva, M.~Halber, T.~Funkhouser, and M.~Niessner.
\newblock Scannet: Richly-annotated 3d reconstructions of indoor scenes.
\newblock In {\em CVPR}, pages 5828--5839, 2017.

\bibitem{engelmann2017exploring}
F.~Engelmann, T.~Kontogianni, A.~Hermans, and B.~Leibe.
\newblock Exploring spatial context for 3d semantic segmentation of point
  clouds.
\newblock In {\em CVPR}, pages 716--724, 2017.

\bibitem{he2017mask}
K.~He, G.~Gkioxari, P.~Doll{\'a}r, and R.~Girshick.
\newblock Mask r-cnn.
\newblock In {\em ICCV}, pages 2980--2988. IEEE, 2017.

\bibitem{huang2018recurrent}
Q.~Huang, W.~Wang, and U.~Neumann.
\newblock Recurrent slice networks for 3d segmentation of point clouds.
\newblock In {\em CVPR}, pages 2626--2635, 2018.

\bibitem{kalogerakis20173d}
E.~Kalogerakis, M.~Averkiou, S.~Maji, and S.~Chaudhuri.
\newblock 3d shape segmentation with projective convolutional networks.
\newblock In {\em CVPR}, pages 6630--6639, 2017.

\bibitem{kingma2014adam}
D.~P. Kingma and J.~Ba.
\newblock Adam: A method for stochastic optimization.
\newblock {\em arXiv preprint arXiv:1412.6980}, 2014.

\bibitem{landrieu2017large}
L.~Landrieu and M.~Simonovsky.
\newblock Large-scale point cloud semantic segmentation with superpoint graphs.
\newblock In {\em CVPR}, pages 4558--4567.

\bibitem{lawin2017deep}
F.~J. Lawin, M.~Danelljan, P.~Tosteberg, G.~Bhat, F.~S. Khan, and M.~Felsberg.
\newblock Deep projective 3d semantic segmentation.
\newblock In {\em International Conference on Computer Analysis of Images and
  Patterns}, pages 95--107. Springer, 2017.

\bibitem{li2018pointcnn}
Y.~Li, R.~Bu, M.~Sun, W.~Wu, X.~Di, and B.~Chen.
\newblock Pointcnn: Convolution on x-transformed points.
\newblock In {\em Advances in Neural Information Processing Systems}, pages
  820--830, 2018.

\bibitem{long2015fully}
J.~Long, E.~Shelhamer, and T.~Darrell.
\newblock Fully convolutional networks for semantic segmentation.
\newblock In {\em CVPR}, pages 3431--3440, 2015.

\bibitem{maturana2015voxnet}
D.~Maturana and S.~Scherer.
\newblock Voxnet: A 3d convolutional neural network for real-time object
  recognition.
\newblock In {\em Intelligent Robots and Systems (IROS), 2015 IEEE/RSJ
  International Conference on}, pages 922--928. IEEE, 2015.

\bibitem{qi2017pointnet}
C.~R. Qi, H.~Su, K.~Mo, and L.~J. Guibas.
\newblock Pointnet: Deep learning on point sets for 3d classification and
  segmentation.
\newblock In {\em CVPR}, pages 652--660, 2017.

\bibitem{qi2017pointnet++}
C.~R. Qi, L.~Yi, H.~Su, and L.~J. Guibas.
\newblock Pointnet++: Deep hierarchical feature learning on point sets in a
  metric space.
\newblock In {\em Advances in Neural Information Processing Systems}, pages
  5099--5108, 2017.

\bibitem{ronneberger2015u}
O.~Ronneberger, P.~Fischer, and T.~Brox.
\newblock U-net: Convolutional networks for biomedical image segmentation.
\newblock In {\em International Conference on Medical image computing and
  computer-assisted intervention}, pages 234--241. Springer, 2015.

\bibitem{shen2018mining}
Y.~Shen, C.~Feng, Y.~Yang, and D.~Tian.
\newblock Mining point cloud local structures by kernel correlation and graph
  pooling.
\newblock In {\em CVPR}, volume~4, 2018.

\bibitem{sinha2016deep}
A.~Sinha, J.~Bai, and K.~Ramani.
\newblock Deep learning 3d shape surfaces using geometry images.
\newblock In {\em ECCV}, pages 223--240. Springer, 2016.

\bibitem{su2015multi}
H.~Su, S.~Maji, E.~Kalogerakis, and E.~Learned-Miller.
\newblock Multi-view convolutional neural networks for 3d shape recognition.
\newblock In {\em CVPR}, pages 945--953, 2015.

\bibitem{tchapmi2017segcloud}
L.~Tchapmi, C.~Choy, I.~Armeni, J.~Gwak, and S.~Savarese.
\newblock Segcloud: Semantic segmentation of 3d point clouds.
\newblock In {\em 3D Vision (3DV), 2017 International Conference on}, pages
  537--547. IEEE, 2017.

\bibitem{uy2018pointnetvlad}
M.~A. Uy and G.~H. Lee.
\newblock Pointnetvlad: Deep point cloud based retrieval for large-scale place
  recognition.
\newblock {\em CVPR}, 2018.

\bibitem{wang2018dynamic}
Y.~Wang, Y.~Sun, Z.~Liu, S.~E. Sarma, M.~M. Bronstein, and J.~M. Solomon.
\newblock Dynamic graph cnn for learning on point clouds.
\newblock {\em ACM Transactions on Graphics}, 2019.

\bibitem{wu20153d}
Z.~Wu, S.~Song, A.~Khosla, F.~Yu, L.~Zhang, X.~Tang, and J.~Xiao.
\newblock 3d shapenets: A deep representation for volumetric shapes.
\newblock In {\em CVPR}, pages 1912--1920, 2015.

\bibitem{ye20183d}
X.~Ye, J.~Li, H.~Huang, L.~Du, and X.~Zhang.
\newblock 3d recurrent neural networks with context fusion for point cloud
  semantic segmentation.
\newblock In {\em ECCV}, pages 415--430. Springer, 2018.

\bibitem{zaheer2017deep}
M.~Zaheer, S.~Kottur, S.~Ravanbakhsh, B.~Poczos, R.~R. Salakhutdinov, and A.~J.
  Smola.
\newblock Deep sets.
\newblock In {\em Advances in Neural Information Processing Systems}, pages
  3391--3401, 2017.

\bibitem{zeng20183dcontextnet}
W.~Zeng and T.~Gevers.
\newblock 3dcontextnet: Kd tree guided hierarchical learning of point clouds
  using local and global contextual cues.
\newblock In {\em ECCV}, pages 314--330, 2018.

\end{thebibliography}
}

\newpage

\begin{appendices}
\section{More Visualizations and Discussions on PS$^2$-Net Variants}
Fig.~\ref{fig:ablation-visualize} shows the visualizations of some qualitative results from the ablation studies - \textbf{``w/o Local''}, \textbf{``w/o EdgeConv''}, \textbf{``w/o NetVLAD''}, and  \textbf{``full PS$^2$-Net"} mentioned in Sec.~4.5 of our main paper. 
We have several interesting findings on the effectiveness of integrating the local and global contexts from these qualitative results:
\begin{itemize}
	\vspace{-0.05in}
	\item The method \textbf{``w/o Local''} (3rd column of Fig.~\ref{fig:ablation-visualize}) shows difficulties in capturing local structures, such as plane or corner, and convex or concave,
	without the exploitation of local context. It can be seen from the results that the method fails to differentiate between ``column'' from ``wall'', which tend to have similar color but different geometric structures.
	Additionally, the segmentation regions are not homogeneous because each point is processed independently in ``w/o Local''. See ``column'' in Row A, ``chair'' in Row B, and ``bookcase'' in Row E. 
	\vspace{-0.05in}
	\item In comparison to ``w/o Local'', \textbf{``w/o EdgeConv''} (4th column of Fig.~\ref{fig:ablation-visualize}) considers neighboring points to some extent. Hence, it performs slightly better in  capturing the geometric differences, which can be seen from the correct classification of ``column'' in Row B and E. However, despite the consideration of local regions, ``w/o EdgeConv'' still treats each point in its local region independently. Consequently, this leads to insufficient exploitation of the fine-grained local structures. As we can see from the wrongly segmented ``chair'' and ``clutter'' (on the ``table'') in Row B and D, ``w/o EdgeConv'' is confused by the neighboring points from the ``bookcase'' class. This indicates that it lacks the ability to capture complex geometry in classes such as ``clutter'' and ``chair''.
	\vspace{-0.05in}
	\item We observe that \textbf{``w/o NetVLAD''} (5th column of Fig.~\ref{fig:ablation-visualize}) fails to model scene-level semantic information without the aggregation of the global context.
	As shown in Row A, ``w/o NetVLAD'' is confused with the ``chairs'' and ``board'' classes, and this does not happen with the global context. Similar confusions can be seen in Row D and F, where the ``wall'' class is confused with the ``table'' and ``chair'' classes, respectively.
	\vspace{-0.05in}
	\item Our \textbf{full PS$^2$-Net} (6th column of Fig.~\ref{fig:ablation-visualize}) that incorporates both local structures and global context achieves the best segmentation results compared to the other methods that remove either the local or global context exploitation module. This suggests that the locally and globally aware framework we proposed is able to capture fine-grained local structures (\eg successful segmentation of ``column'' class in Row A-F, distinction between ``clutter'' and ``bookcase'' classes in Row B and D) as well as the global context (\eg distinction between the white ``wall'' and white ``board'' classes in Row B,C,D,F)
\end{itemize}

\begin{figure*}[h]
	\centering
	\includegraphics[scale=0.68]{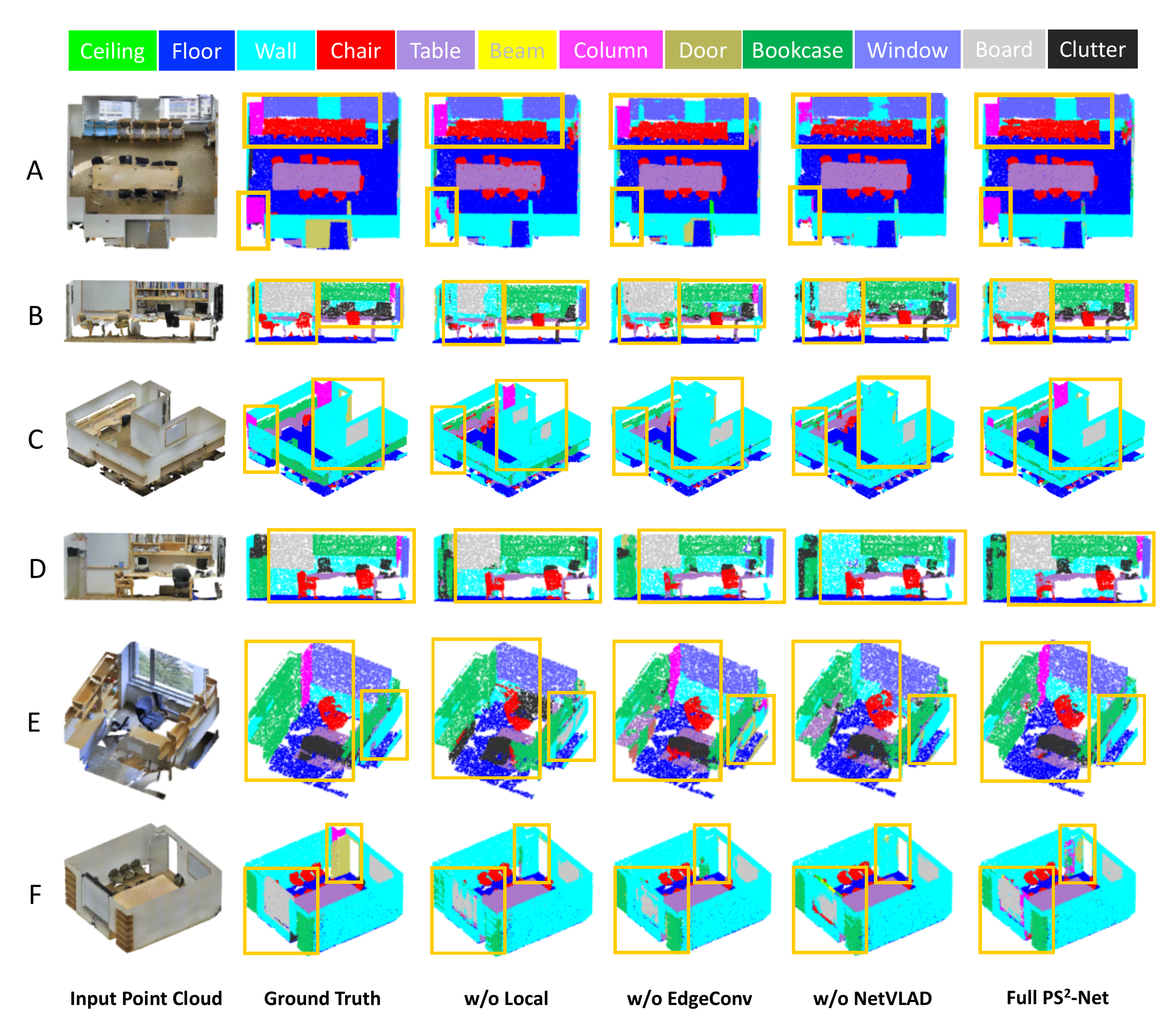}
	\caption{Visualized examples for ablation test of PS$^2$-Net variants on S3DIS [A5]. The interesting areas are highlighted by yellow bounding boxes. Best viewed in color.}
	\label{fig:ablation-visualize}
\end{figure*}

\section{More Discussions on Unbalance Dataset Problem}
The two datasets used in our experiments are very challenging as both of them are highly unbalanced. The statistics of data portion of the two datasets are given in Table~\ref{tbl:s3dis-statistics} and Table~\ref{tbl:scannet-statistics}, respectively.
As seen in Table~\ref{tbl:s3dis-statistics}, ``wall'', ``ceiling'', ``floor'' that are the dominant classes on S3DIS have 10 $\sim$ 66 times more data than the rare classes, such as  ``sofa'', ``board'', and ``beam''. This unbalance problem is more severe on the ScanNet dataset, where the most dominate class (\ie ``wall'') has over 200 times more training data than the rarest class (\ie ``shower curtain'').  
In the experiments, we did not implement any strategy (\eg over-sampling and weighted loss based on class frequency) to explicitly solve this problem. Interestingly, we still achieve acceptable performance on the rarest classes, such as ``column'', ``window'', ``board'' on the S3DIS dataset, and ``bathtub'', ``toilet'', ``window'', ``shower curtain'' on the ScanNet dataset. This may indicate that our proposed method is robust to unbalanced data distributions.

\begin{table*}[h]
	\centering
	\caption{Statistics of class-wise instance number and point portion on S3DIS dataset.}
	\scalebox{0.85}{
		\begin{tabular}{l|c c c c c c c c c c c c c}
			\hline\toprule[1pt]
			\textbf{Class name}& \textbf{ceiling} & \textbf{floor} &\textbf{wall} & \textbf{beam}& \textbf{column}& \textbf{window}& \textbf{door}& \textbf{table}& \textbf{chair}& \textbf{sofa}  &  \textbf{bookcase}&\textbf{board} & \textbf{clutter} \\\hline\toprule[0.5pt]
			\textbf{\# object} & 391 & 290 & 1,552 & 165 & 260 & 174 & 549 & 461 & 1,369 & 61 & 590 & 143 & ---\\\hline
			\textbf{Data percentage(\%)}  & 19.27 & 16.52 & 27.81 & 1.73 & 2.02 & 2.52 & 4.78 & 3.39 & 3.43 & 0.42 & 6.33 & 1.24 & 10.54 \\\hline\toprule[1pt]
		\end{tabular}
	}
	\label{tbl:s3dis-statistics}
\end{table*}

\begin{table*}[h]
	\centering
	\caption{Statistics of class-wise point portion on ScanNet dataset.}
	\scalebox{0.78}{
		\begin{tabular}{p{4cm} | p{1.2cm} p{1.2cm} p{1cm} p{1.2cm} p{1.2cm} p{1cm} p{2cm} p{1cm} p{1cm} p{1.5cm} p{1cm}}
			\hline\toprule[1pt]
			\textbf{Class name} & \textbf{wall} & \textbf{floor} &\textbf{chair} & \textbf{table}& \textbf{desk}& \textbf{bed}& \textbf{bookshelf}& \textbf{sofa}& \textbf{sink}  &\textbf{bathtub} & \textbf{toilet} \\\hline\toprule[0.5pt]
			\textbf{Train data percentage(\%)}  & 36.8 & 24.90 & 4.60 & 2.53 & 1.66 & 2.58 & 2.04 & 2.59 & 0.34 & 0.34 & 0.27\\\hline
			\textbf{Test data percentage(\%)}  & 36.46 & 24.38 & 5.36 & 2.90 & 1.65 & 2.14 & 2.15 & 2.44 & 0.33 & 0.22 & 0.26  \\\hline\toprule[1pt]
	\end{tabular}}
	\scalebox{0.78}{
		\begin{tabular}{p{4cm}| p{1.2cm} p{1.2cm} p{1cm} p{1.2cm} p{2.4cm} p{2cm} p{1cm} p{1.1cm} p{3cm}}
			\hline\toprule[1pt]
			\textbf{Class name} & \textbf{curtain} & \textbf{counter} & \textbf{door} & \textbf{window} & \textbf{shower curtain} & \textbf{refridgerator} & \textbf{picture} & \textbf{cabinet} & \textbf{other furniture} \\\hline\toprule[0.5pt]
			\textbf{Train data percentage(\%)} & 1.48 & 0.62 & 2.33 & 0.94 & 0.18 & 0.43 & 0.37 & 2.59 & 2.46\\\hline
			\textbf{Test data percentage(\%)}  & 0.98 & 0.65 & 2.15 & 0.66 & 0.10 & 0.34 & 0.17 & 2.43 & 3.23\\\hline\toprule[1pt]
	\end{tabular}}
	\label{tbl:scannet-statistics}
\end{table*}

\end{appendices}

\end{document}